\pdfoutput=1

\documentclass[10pt]{article}

\usepackage{Sty/mcr}
\usepackage{Sty/pkg}
\usepackage{Sty/thmE}
\usepackage{subcaption}

\usepackage{Sty/algorithm} 
\usepackage{Sty/algorithmic}

\usepackage{Sty/cancel}

\usepackage{Sty/soul}

\graphicspath{ {./img/} }

\usepackage{hyperref}

\usepackage{Sty/bbm}

\usepackage[margin=1.2in]{Sty/geometry}
\usepackage{Sty/natbib}

\title{Valid $P$-Value for Deep Learning-Driven Salient Region}

\date{\today}

\author{
    Daiki Miwa\thanks{Equal contribution} \\
    Nagoya Institute of Technology\\
    miwa.daiki.mllab.nit@gmail.com
   \and
    Vo Nguyen Le Duy$^\ast$\\
    RIKEN \\
    duy.mllab.nit@gmail.com
  \and
    Ichiro Takeuchi\thanks{Corresponding author} \\
    Nagoya University and RIKEN\\
    ichiro.takeuchi@mae.nagoya-u.ac.jp
}

\begin{document}

\maketitle

\begin{abstract}
Various saliency map methods have been proposed to interpret and explain predictions of deep learning models. 
Saliency maps allow us to interpret which parts of the input signals have a strong influence on the prediction results. 
However, since a saliency map is obtained by complex computations in deep learning models, it is often difficult to know how reliable the saliency map itself is. 
In this study, we propose a method to quantify the reliability of a salient region in the form of $p$-values. 
Our idea is to consider a salient region as a selected hypothesis by the trained deep learning model and employ the selective inference framework. 
The proposed method can provably control the probability of false positive detections of salient regions. 
We demonstrate the validity of the proposed method through numerical examples in synthetic and real datasets. 
Furthermore, we develop a Keras-based framework for conducting the proposed selective inference for a wide class of CNNs without additional implementation cost.

\end{abstract}

\clearpage

% --------------- Main Text --------------------

\section{Introduction}
Deep neural networks (DNNs) have exhibited remarkable predictive performance in numerous practical applications in various domains 
owing to their ability to automatically discover the representations needed for prediction tasks from the provided data.
To ensure that the decision-making process of DNNs is transparent and easy to understand, it is crucial to effectively explain and interpret DNN representations.
For example, in image classification tasks, obtaining \emph{salient regions} allows us to explain which parts of the input image strongly influence the classification results.

Several saliency map methods have been proposed to explain and interpret the predictions of DNN models \citep{ribeiro2016should, bach2015pixel, doshi2017towards, lundberg2017unified, zhou2016learning, selvaraju2017grad}.
However, the results obtained from saliency methods are fragile \citep{kindermans2017reliability, ghorbani2019interpretation, melis2018towards, zhang2020interpretable, dombrowski2019explanations, heo2019fooling}.
It is important to develop a method for quantifying the reliability of DNN-driven salient regions.

Our idea is to interpret salient regions as hypotheses driven by a trained DNN model and employ a statistical hypothesis testing framework.
We use the $p$-value as a criterion to quantify the statistical reliability of the DNN-driven hypotheses.
Unfortunately, constructing a valid statistical test for DNN-driven salient regions is challenging because of the \emph{selection bias}. 
In other words, because the trained DNN \emph{selects} the salient region based on the provided data, the post-selection assessment of importance is biased upwards.

To correct the selection bias and compute valid $p$-values for DNN-driven salient regions, we introduce a conditional \emph{selective inference} (SI) approach.
The selection bias is corrected by conditional SI in which the test statistic conditional on the event that the hypotheses (salient regions) are selected using the trained DNNs.
Our main technical contribution is to develop a computational method for explicitly deriving the exact (non-asymptotic) conditional sampling distribution of the salient region for a wide class convolutional neural networks (CNNs), which enables us to conduct conditional SI and compute valid $p$-values.
Figure \ref{fig:intro} presents an example of the problem setup.

\begin{figure}[!t]
\centering
\begin{subfigure}{\textwidth}
  \centering
  \includegraphics[width=\linewidth]{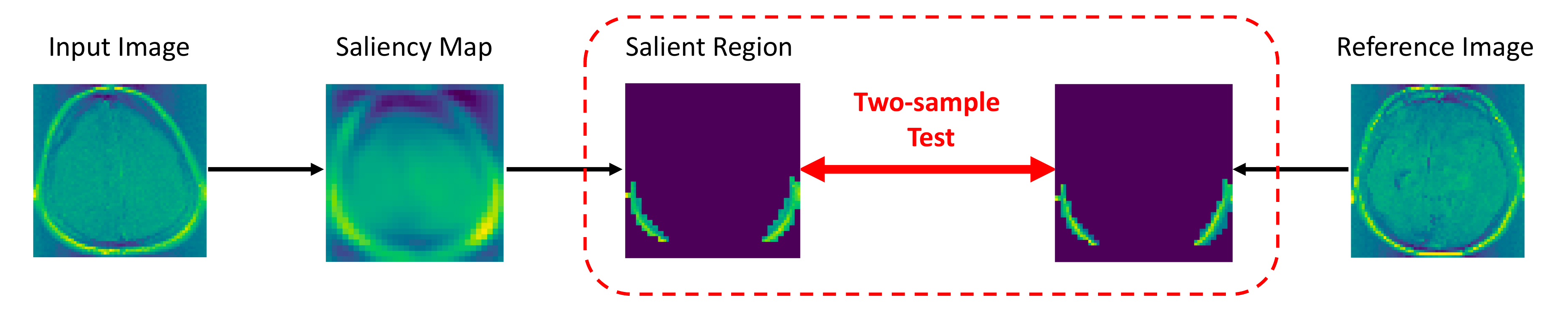} 
  \caption{Image without tumor region. The naive-$p$ = \textbf{0.00} (wrong detection) and selective-$p$ = \textbf{0.43} (true negative)}
\end{subfigure}
\hspace{4pt}
\begin{subfigure}{\textwidth}
  \centering
  \includegraphics[width=\linewidth]{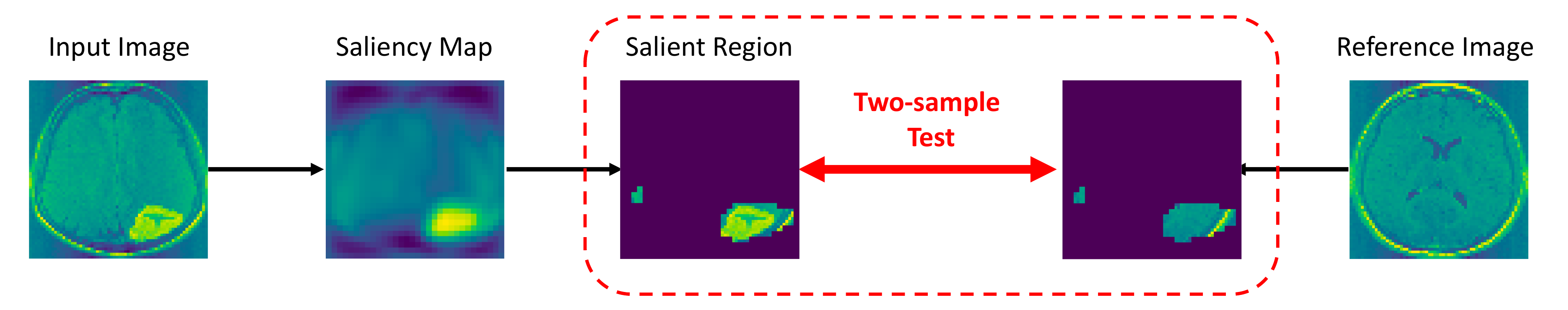}  
  \caption{Image with tumor region. The naive-$p$ = \textbf{0.00} (true positive) and selective-$p$ = \textbf{0.00} (true positive)}
\end{subfigure}
\caption{
Examples of the problem setup and the proposed method on brain tumor dataset.
By applying a saliency method called CAM \citep{zhou2016learning} on a query input image, we obtain the salient region.
Our goal is to provide the statistical significance of the salient region in the form of $p$-value by considering two-sample test between the salient region and the corresponding region in the a reference image.
Note that, since the salient region is selected based on the data, the degree of saliency in the selected region is biased upward.
In the upper image where there is no true brain tumor, the naive $p$-value which is obtained without caring the selection bias is nearly zero, indicating the false positive finding of the salient region.
On the other hand, the selective $p$-value which is obtained by the proposed conditional SI approach is 0.43, indicating that the selected saliency region is not statistically significant.
In the lower figure where there is a true brain tumor, both the naive $p$-value and the selective $p$-value are very small, indicating true positive finding.
These results illustrate that naive $p$-value cannot be used to quantify the reliability of DNN-based salient region.
In contrast, with the selective $p$-values, we can successfully identify false positive and true positive detections with a desired error rate.
}
\label{fig:intro}
\end{figure}

\paragraph{Related works.}

In this study, we focus on statistical hypothesis testing for post-hoc analysis, i.e., quantifying the statistical significance of the salient regions identified in a trained DNN model when a test input instance is fed into the model. 
Several methods have been developed to visualize and understand trained DNNs.
Many of these post-hoc approaches \citep{mahendran2015understanding, zeiler2014visualizing, dosovitskiy2016inverting, simonyan2013deep} have focused on developing visualization tools for saliency maps given a trained DNN. 
Other methods have aimed to identify the discriminative regions in an input image given a trained network \citep{selvaraju2017grad, fong2017interpretable, zhou2016learning, lundberg2017unified}.
Furthermore, some recent studies have shown that many popular methods for explanation and interpretation are not stable against a perturbation or adversarial attack on the input data and model~\citep{kindermans2017reliability, ghorbani2019interpretation, melis2018towards, zhang2020interpretable, dombrowski2019explanations, heo2019fooling}.
However, to the best of our knowledge, no study to date has quantitatively evaluated and reproducibility of DNN-driven salient regions with a rigorous statistical inference framework.

Recently, conditional SI has been recognized as a promising new approach for evaluating the statistical significance of data-driven hypotheses. Conditional SI has been mainly studied for inference of linear model features selected by a feature selection method such as Lasso \citep{lee2016exact, liu2018more,hyun2018exact, le2021parametric} and stepwise feature selection \citep{tibshirani2016exact,sugiyama2020more}. 
The main idea of conditional SI study is to make inferences conditional on selection events, which allows us to derive exact sampling distributions of test statistics. 
In addition, conditional SI has been applied to various problems \citep{fithian2015selective, tian2018selective, yang2016selective, hyun2021post, duy2020computing, sugiyama2021valid, chen2019valid, panigrahi2016bayesian, tsukurimichi2021conditional, hyun2018exact, tanizaki2020computing, duy2021more, tibshirani2016exact, sugiyama2020more, suzumura2017selective, das2021fast, duy2022exact}.

Most relevant existing work of this study is \citet{duy2020quantifying}, where the authors provide a framework for computing valid $p$-values for DNN-based image segmentation results.
In this paper, we generalized this work so that hypotheses characterized by any internal nodes of the network can be considered, enabling us to quanfity the statistical significance of salient regions. 
This is in contrast to \citet{duy2020quantifying}'s work, which only considered the inference of the DNN's output in a segmentation task. 
Furthermore, we introduce a Keras-based implementation framework that enables us to conduct SI for a wide class of CNNs without additional implementation costs. 
This is in contrast to \citet{duy2020quantifying}'s work, where the selection event must be implemented whenever the network architecture is changed.
In another direction, \citet{burns2020interpreting} considered the black box model interpretability as a multiple-hypothesis testing problem.
They aimed to deduce important features by testing the significance of the difference between the model prediction and what would be expected when replacing the features with their counterfactuals.
The difficulty of this multiple-hypothesis testing approach is that the number of hypotheses to be considered is large (e.g., in the case of an image with $n$ pixels, the number of possible salient regions is $2^n$).
Multiple testing correction methods, such as the Bonferroni correction, are highly conservative when the number of hypotheses is large. 
To circumvent this difficulty, they only considered a tractable number of regions selected by a human expert or object detector, which causes selection bias because these candidate regions are selected based on the data.

\paragraph{Contribution.} Our main contributions are as follows:

$\bullet$ We provide an exact (non-asymptotic) inference method for salient regions based on the SI concept. To the best of our knowledge, this is the first method that proposes to provide valid $p$-values to statistically quantify the reliability of DNN-driven salient regions.

$\bullet$ We propose a novel algorithm and its implementation. Specifically, we propose Keras-based implementation enables us to conduct conditional SI for a wide class of CNNs without additional implementation costs.

$\bullet$ We conducted experiments on both synthetic and real-world datasets, through which we show that our proposed method can successfully control the false positive rate, has good performance in terms of computational efficiency, and provides good results in practical applications.
We provide the detailed description of our implementation in the supplementary document. Our code is available at
\begin{center}
\href{https://github.com/takeuchi-lab/selective_inference_dnn_salient_region}{https://github.com/takeuchi-lab/selective\_inference\_dnn\_salient\_region}.
\end{center}

\section{Problem Formulation} \label{sec:problem_setup}
In this paper, we consider the problem of quantifying the statistical significance of the salient regions identified by a trained DNN model when a test input instance is fed into the model.
Consider an $n$-dimensional \emph{query} input vector
\begin{align*} %\label{eq:random_data}
 \bm X &= (X_1, ..., X_n)^\top = \bm s + \bm \veps, ~~~ \bm \veps \sim \NN(\bm 0, \sigma^2I_n)
\end{align*}
and an $n$-dimensional \emph{reference} input vector, 
\begin{align*} %\label{eq:random_data_ref}
 \bm X^{\rm ref} &= (X^{\rm ref}_1, ..., X^{\rm ref}_n)^\top = \bm s^{\rm ref} + \bm \veps^{\rm ref}, ~~~ \bm \veps^{\rm ref} \sim \NN(\bm 0, \sigma^2I_n),
\end{align*}
where
$\bm s, \bm s^{\rm ref} \in \RR^n$
are the signals
and 
$\bm \veps, \bm \veps^{\rm ref} \in \RR^n$
are the noises
for query and reference input vectors, respectively. 
We assume that the signals, 
$\bm s$ and $\bm s^{\rm ref}$
are unknown,
whereas the distribution of noises 
$\bm \veps$ and $\bm \veps^{\rm ref}$
are known (or can be estimated from external independent data) to follow 
$\NN(\bm 0, \sigma^2 I_n)$,
an $n$-dimensional normal distribution with a mean vector $\bm 0$ and covariance matrix $\sigma^2 I_n$, which are mutually independent.
In the illustrative example presented in \S1,
$\bm X$
is a query brain image for a potential patient (we do not know whether she/he has a brain tumor),
whereas
$\bm X^{\rm ref}$
is a brain image of a healthy person known to be without brain tumors.

Consider a saliency method for a trained CNN.
We denote the saliency method as a function
$\cA: \RR^n \to \RR^n$
that takes a query input vector
$\bm X \in \RR^n$
and returns the saliency map
$\cA(\bm X) \in \RR^n$.
We define a \emph{salient region} $\cM_{\bm X}$ for the query input vector $\bm X$ as the set of elements whose saliency map value is greater than a threshold
\begin{align}
 \label{eq:threshold}
 \cM_{\bm X} = \left \{ i\in [n] : \cA_i(\bm X) \geq \tau \right \},
\end{align}
where $\tau \in \RR$ denotes the given threshold.
In this study, we consider CAM~\citep{zhou2016learning} as an example of saliency method and threshold-based definition of the salient region. Our method can be applied to other saliency methods and other definition of salient region. 

\paragraph{Statistical inference.}
To quantify the statistical significance of the saliency region $\cM_{\bm X}$, we consider such \emph{two-sample test} to quantify the statistical significance of the difference between the salient regions of the query input vector
$\bm X_{\cM_{\bm X}}$
and corresponding region of the reference input vector
$\bm X^{\rm ref}_{\cM_{\bm X}}$.
As concrete examples of the two-sample test, we consider the \emph{mean null test}:
\begin{align}
 \label{eq:mean_null}
 {\rm H}_{0} : 
 \frac{1}{|\cM_{\bm X}|}\sum \limits_{i \in \cM_{\bm X}} s_i = \frac{1}{|\cM_{\bm X}|} \sum \limits_{i \in \cM_{\bm X}} s_i^{\rm ref}
 \quad 
 \text{v.s.}
 \quad  
 {\rm H}_{1} : 
 \frac{1}{|\cM_{\bm X}|} \sum \limits_{i \in \cM_{\bm X}} s_i \neq \frac{1}{|\cM_{\bm X}|} \sum \limits_{i \in \cM_{\bm X}} s_i^{\rm ref}.
\end{align}
and \emph{global null test}: 
\begin{align}
 \label{eq:global_null}
 {\rm H}_{0} : s_i = s_i^{\rm ref},  ~ \forall i \in \cM_{\bm X},
 \quad 
 \text{v.s.}
 \quad  
 {\rm H}_{1} : s_i \neq s_i^{\rm ref},  ~ \exists i \in \cM_{\bm X},
\end{align}
In the mean null test depicted in Eq. \eq{eq:mean_null}, we consider a null hypothesis that the average signals in the salient region $\cM_{\bm X}$ are the same between $\bm X$ and $\bm X^{\rm ref}$.
In contrast, in the global null test in Eq. \eq{eq:global_null}, we consider a null hypothesis that all elements of the signals in the salient region $\cM_{\bm X}$ are the same between $\bm X$ and $\bm X^{\rm ref}$.
The $p$-values for these two-sample tests can be used to quantify the statistical significance of the salient region $\cM_{\bm X}$.

\paragraph{Test-statistic.}
For a two-sample test conducted between
$\bm X_{\cM_{\bm X}}$ and $\bm X^{\rm ref}_{\cM_{\bm X}}$, we consider a class of test statistics called 
\emph{conditionally linear test-statistic}, which is expressed as 
\begin{align}
 \label{eq:conditoinally_linear_statistic}
 T(\bm X, \bm X^{\rm ref}) = \bm \eta^\top_{\cM_{\bm X}} {\bm X \choose \bm X^{\rm ref} }, ~ 
\end{align}
and \emph{conditionally $\chi$ test-statistic},
which is expressed as 
\begin{align}
    \label{eq:conditionally_chi_satistic}
    T(\bm X, \bm X^{\rm ref}) =  \sigma^{-1}\left\lVert P _{\cM_{\bm X}} {\bm X \choose \bm X^{\rm ref} }\right\rVert,
\end{align}
where
$\bm \eta_{\cM_{\bm X}} \in \RR^{2n}$ is a vector and $P _{\cM_{\bm X}} \in \RR^{2n\times2n}$
is a projection matrix that depends on saliency region $\cM_{\bm X}$.

% Note that the test statistic
% $T(\bm X, \bm X^{\rm ref})$
% is \emph{not} linear in
% $(\bm X ~ \bm X^{\rm ref})^\top$
% because
% $\bm \eta_{\cM_{\bm X}}$ 
% depends on
% $\cM_{\bm X}$.

% We state that the test statistic 
% $T(\bm X, \bm X^{\rm ref})$
% is \emph{ conditionally} linear 
% because it is linear in 
% $(\bm X ~ \bm X^{\rm ref})^\top$ 
% when
% the salient region
% $\cM_{\bm X}$
% is fixed. 

The test statistics for the mean null tests and the global null test can be written in the form of Eq. \eq{eq:conditoinally_linear_statistic} and \eq{eq:conditionally_chi_satistic}, respectivery.
For the mean null test in Eq. \eq{eq:mean_null}, we consider the following test-statistic
\begin{align*} %\label{eq:test_statistic_2}
 T(\bm X, \bm X^{\rm ref}) 
 = \bm \eta^{\top}_{\cM_{\bm X}} { \bm X \choose \bm X^{\rm ref}}
 = \frac{1}{|\cM_{\bm X}|} \sum_{i \in \cM_{\bm X}}
 X_i
 -  
 \frac{1}{|\cM_{\bm X}|} \sum_{i \in \cM_{\bm X}}
 X^{\rm ref}_i,
\end{align*}
where 
$
\bm \eta_{\cM_{\bm X}}
= \frac{1} {|\cM_{\bm X}|} 
\begin{pmatrix}
 \mathbf{1}^n_{\cM_{\bm X}} \\ 
 - \mathbf{1}^n_{\cM_{\bm X}} 
\end{pmatrix}
\in \RR^{2n}.
$
For the gloabl null test in Eq. \eq{eq:global_null}, we consider the following test-statistic
\begin{align*} %\label{eq:test_statistic_1}
    T(\bm X, \bm X^{\rm ref})
    = \sigma^{-1} \left\lVert P_{\cM_{\bm X}}{\bm X \choose \bm X^{\rm ref}} \right\rVert
 =  \sqrt{
 \sum \limits_{i \in \cM_{\bm X}}
 \left ( \frac{X_i - X_i^{\rm ref}}{\sqrt{2}\sigma}  \right )^2 
 } ,
\end{align*}
where 
\begin{align}
P_{\cM_{\bm X}} 
= \frac{1}{2}
\left(
    \begin{matrix}
    {\rm diag}(\bm 1^n_{\cM_{\bm X}}) & -{\rm diag}(\bm 1^n_{\cM_{\bm X}})\\
    -{\rm diag}(\bm 1^n_{\cM_{\bm X}}) & {\rm diag}(\bm 1^n_{\cM_{\bm X}})
    \end{matrix}
\right).
\end{align}  

To obtain $p$-values for these two-sample tests we need to know the sampling distribution of the test-statistics.
Unfortunately, it is challenging to derive the sampling distributions of test-statistics because they depend on the salient region $\cM_{\bm X}$, which is obtained through a complicated calculation in the trained CNN.

\section{Computing Valid $p$-value by Conditional Selective Inference} \label{sec:proposed_method}

In this section, we introduce an approach to compute the valid $p$-values for the two-sample tests for the salient region $\cM_{\bm X}$ between the query input vector $\bm X$ and the reference input vector $\bm X^{\rm ref}$ based on the concept of conditional SI \citep{lee2016exact}.

\subsection{Conditional Distribution and Selective $p$-value}

\paragraph{Conditional distribution.}
The basic idea of conditional SI is to consider the sampling distribution of the test-statistic conditional on a \emph{selection event}. 
Specifically, we consider the sampling property of the following conditional distribution
\begin{align} \label{eq:conditional_inference}
    T(\bm{X},\bm{X}^{\text{ref}})
 ~ \Big | ~
 \left \{ \cM_{\bm X} = \cM_{\bm X_{\rm obs}} \right \}, 
\end{align}
where
$\bm X_{\rm obs}$ 
is the observation (realization) of random vector $\bm X$.
The condition in Eq.\eq{eq:conditional_inference} indicates the randomness of $\bm X$ conditional on the event that the same salient region $\cM_{\bm X}$ as the observed $\cM_{\bm X^{\rm obs}}$ is obtained.
By conditioning on the salient region $\cM_{\bm X}$, derivation of the sampling distribution of the conditionally linear and $\chi$ test-statistic
$T(\bm X, \bm X^{\rm ref})$
is reduced to a derivation of the distribution of linear function and quadratic function of
$(\bm X, \bm X^{\rm ref})$, respectively.

\paragraph{Selective $p$-value.}
After considering the conditional sampling distribution in \eq{eq:conditional_inference}, we introduce the following \emph{selective $p$-value}:
\begin{align} \label{eq:selective_p}
	p_{\rm selective} 
	= \mathbb{P}_{\rm H_0} 
	\Big ( 
	\left |T(\bm X, \bm X^{\rm ref}) \right | 
	\geq 
	\left |T(\bm X_{\rm obs}, \bm X^{\rm ref}_{\rm obs}) \right | 
	~ \Big | ~
	\cM_{\bm X} = \cM_{\bm X_{\rm obs}}, ~
	\cQ_{\bm X, \bm X^{\rm ref}} = \cQ_{\rm obs}
	\Big ),
\end{align}
where 
\begin{align*} %\label{eq:nuisance}
 \cQ_{\bm X, \bm X^{\rm ref}} 
 = \Omega_{\bm X, \bm X^{\rm ref}}, 
 \quad 
 \cQ_{\rm obs} = \cQ_{\bm X_{\rm obs}, \bm X^{\rm ref}_{\rm obs}}
\end{align*}
with 
\begin{align*}
     \Omega_{\bm X, \bm X^{\rm ref}}=
     \left (
     I_{2n} - 
     \frac{
     \bm \eta_{\cM_{\bm X}}  \bm \eta_{\cM_{\bm X}}^\top
     } { 
     \lVert \bm{\eta}_{\cM_{\bm X}}  \rVert^2
     } 
     \right ) 
     { \bm X \choose \bm X^{\rm ref}} \in \RR^{2n}
\end{align*}
in the case of mean null test, and
\begin{align*}
 \cQ_{\bm X, \bm X^{\rm ref}} 
 = 
 \left\{
     \mathcal{V}_{\bm X, \bm X^{\rm ref}}
     ,\;
     \mathcal{U}_{\bm X, \bm X^{\rm ref}}
 \right\},
 \quad 
 \cQ_{\rm obs} = \cQ_{\bm X_{\rm obs}, \bm X^{\rm ref}_{\rm obs}}
 \quad
\end{align*}
with 
\begin{align*}
    \mathcal{V}_{\bm X, \bm X^{\rm ref}} = 
    \sigma P_{\cM_{\bm X}} { \bm X \choose \bm X^{\rm ref}} \Big 
    / {\left\lVert P_{\cM_{\bm X}} { \bm X \choose \bm X^{\rm ref}} \right\rVert} \in \RR^{2n}
    ,\quad
    \mathcal{U}_{\bm X, \bm X^{\rm ref}} =
    P_{\cM_{\bm X}}^\perp { \bm X \choose \bm X^{\rm ref}}\in \RR^{2n}
\end{align*}
in the case of global null test.
The $\cQ_{\bm X, \bm X^{\rm ref}}$ is the sufficient statistic of the nuisance parameter that needs to be conditioned on in order to tractably conduct the inference~\footnote{
This nuisance parameter $\cQ_{\bm X, \bm X^{\rm ref}}$ corresponds to the component $\bm z$ in the seminal conditional SI paper \citep{lee2016exact} (see Sec. 5, Eq. 5.2 and Theorem 5.2) and $\bm{z},\bm{w}$ in \citep{chen2019valid}(see Sec. 3, Theorem 3.7).
We note that additional conditioning on $\cQ_{\bm X, \bm X^{\rm ref}}$ is a standard approach in the conditional SI literature and is used in almost all the conditional SI-related studies.
Here, we would like to note that the selective $p$-value depend on $\cQ_{\bm X, \bm X^{\rm ref}}$, but the property in \eq{eq:sampling_property_selective_p} is satisfied without this additional condition because we can marginalize over all values of $\cQ_{\bm X, \bm X^{\rm ref}}$ (see the lower part of the proof of Theorem 5.2 in \citet{lee2016exact} and the proof of Theorem 3.7 in \citet{chen2019valid} ).
}.

The selective $p$-value in Eq.\eq{eq:selective_p} has the following desired sampling property 
\begin{align} \label{eq:sampling_property_selective_p}
	\PP_{{\rm H}_0} \Big (p_{\rm selective} \leq \alpha 
	\mid
    \cM_{\bm X} = \cM_{\bm X_{\rm obs}}
	\Big ) = \alpha, \quad \forall \alpha \in [0, 1].
\end{align}
This means that the selective $p$-values $p_{\rm selective}$ can be used as a valid statistical significance measure for the salient region $\cM_{\bm X}$. 

\subsection{Characterization of the Conditional Data Space}
To compute the selective $p$-value in \eq{eq:selective_p}, we need to characterize the conditional data space whose characterization is described introduced in the next section.
We define the set of $(\bm X ~ \bm X^{\rm ref})^\top \in \RR^{2n}$ that satisfies the conditions in Eq. \eq{eq:selective_p} as 
\begin{align} \label{eq:conditional_data_space}
	\cD = 
	\left \{ 
		(\bm X ~ \bm X^{\rm ref})^\top \in \RR^{2n}
		 ~ \big |  ~ 
        \cM_{\bm X} = \cM_{\bm X_{\rm obs}},
		\cQ_{\bm X, \bm X^{\rm ref}} = \cQ_{\rm obs}
	\right \}. 
\end{align}
According to the second condition, the data in $\cD$ is restricted to a line in $\RR^{2n}$ as stated in the following Lemma.
\begin{lemma} \label{lemma:data_line}
Let us define
let us define,
\begin{align} \label{eq:a_b_line}
	\bm a = 
    \Omega_{\bm X_{\rm obs}, \bm X^{\rm ref}_{\rm obs}}
	\quad \text{and} \quad 
    \bm b = \frac{\bm \eta_{\cM_{\bm X}}} {\lVert {\bm \eta_{\cM_{\bm X}}} \rVert^2} \in \RR^{2n}.
\end{align}
in the mean null test, and
\begin{align} \label{eq:a_b_line}
    \bm a = \mathcal{U}_{\bm X_{\rm obs}, \bm X^{\rm ref}_{\rm obs}}
	\quad \text{and} \quad 
    \bm b = \mathcal{V}_{\bm X_{\rm obs}, \bm X^{\rm ref}_{\rm obs}}
\end{align}
in the case of global null test.
Then, the set $\cD$ in \eq{eq:conditional_data_space} can be rewritten as 
$
\cD = \Big \{ \big (\bm X ~ \bm X^{\rm ref} \big )^\top = \bm a + \bm b z \mid z \in \cZ \Big \}
$
by using the scalar parameter $z \in \RR$, where
\begin{align} \label{eq:cZ}
	\cZ = \left \{ 
	z \in \RR ~
	\mid  
    \cM_{\bm a_{1:n} + \bm b_{1:n}} z = \cM_{\bm X_{\rm obs}}
	\right \}.
\end{align}

$\bm{x}_{1:n}$ represents a vector of elements $1$ through $n$ of $\bm x$.
\end{lemma}

\begin{proof}
The proof is deferred to Appendix \ref{app:proof_lemma_1}
\end{proof}

Lemma \ref{lemma:data_line} indicates that we do not need to consider the $2n$-dimensional data space.
Instead, we only need to consider the \emph{one-dimensional projected} data space $\cZ$ in \eq{eq:cZ}.
Now, let us consider a random variable $Z \in \RR$ and its observation $Z_{\rm obs} \in \RR$ that satisfies $(\bm X ~ \bm X^{\rm ref})^\top = \bm a + \bm b Z$ and $(\bm X_{\rm obs} ~ \bm X^{\rm ref}_{\rm obs})^\top = \bm a + \bm b Z_{\rm obs}$.
The selective $p$-value (\ref{eq:selective_p}) is rewritten as 
\begin{align} \label{eq:selective_p_parametrized}
	p_{\rm selective} 
	 = \mathbb{P}_{\rm H_0} \left ( |Z| \geq |Z_{\rm obs}| 
	\mid 
	Z \in \cZ
	\right).
\end{align}
Because the variable $Z \sim \NN(0, \sigma^2\lVert \bm{\eta} \rVert^2 )$ in the case of mean null test and $Z \sim \chi\left( \mathrm{Trace}(P) \right) $ in the case of global null test under the null hypothesis, $Z \mid Z \in \cZ$ follows a \emph{truncated} normal distribution and a \emph{truncated} $\chi$ distribution, respectively.
Once the truncation region $\cZ$ is identified, computation of the selective $p$-value in (\ref{eq:selective_p_parametrized}) is straightforward.
Therefore, the remaining task is to identify $\cZ$.

In general, computation of $\cZ$ in \eq{eq:cZ} is difficult because we need to identify the selection event $\cM_{\bm a_{1:n} + \bm b_{1:n} z}$ for all values of $z \in \RR$, which is computationally challenging.
In the next section, we show that the challenge can be resolved under a wide class of problems.
% ============================

\section{Piecewise Linear Network} \label{sec:class_problem}
The problem of computing selective $p$-values for the selected salient region is casted into the problem of identifying a set of intervals $\cZ = \{z \in \RR \mid \cM_{\bm X (z)} = \cM_{\bm{X}_{\rm{obs}}}\}$.
Given the complexity of saliency computation in a trained DNN, it seems difficult to obtain $\cZ$.
In this section, however, we explain that this is feasible for a wide class of CNNs.

\paragraph{Piecewise linear components in CNN}
The key idea is to note that most of basic operations and common activation functions used in a trained CNN can be represented as piecewise linear functions in the following form:
% ============================
\begin{definition} \label{def:piecewise_linear}
(Piecewise Linear Function) A piecewise linear function $f : \RR^{n} \mapsto \RR^{m}$ is defined as:
\begin{align*} %\label{eq:piecewise_linear_function}
	f(\bm X) = 
	\begin{cases}
		\Psi_1^f \bm X + \bm \psi_1^f, & \text{if} ~~  \bm X \in \cP^f_1 := \{\bm X^\prime \in \RR^n \mid \Delta_1^f \bm X^\prime \le \bm \delta_1^f\},\\ 
		\Psi_2^f \bm X + \bm \psi_2^f, & \text{if} ~~  \bm X \in \cP^f_2 := \{\bm X^\prime \in \RR^n \mid \Delta_2^f \bm X^\prime \le \bm \delta_2^f\},\\ 
		\quad \quad \quad  \vdots \\ 
		\Psi_{K(f)}^f \bm X + \bm \psi_{K(f)}^f, & \text{if} ~~  \bm X \in \cP^f_{K(f)} := \{\bm X^\prime \in \RR^n \mid \Delta_{K(f)}^f \bm X^\prime \le \bm \delta_{K(f)}^f\}
	\end{cases}
\end{align*}
where $\Psi_k^f$, $\bm \psi_k^f$, $\Delta_k^f$ and $\bm \delta_k^f$ for $k \in [K(f)]$ are certain matrices and vectors with appropriate dimensions, $\cP_k^f := \{\bm x \in \RR^n \mid \Delta_k^f \bm x \le \bm \delta_k^f\}$ is a polytope in $\RR^n$ for $k \in [K(f)]$, and $K(f)$ is the number of polytopes for the function $f$.
\end{definition}

Examples of piecewise linear components in a trained CNN are shown in Appendix \ref{app:example_piecewise_linear}.

\paragraph{Piecewise Linear Network}
\begin{definition} (Piecewise Linear Network)
 A network obtained by concatenations and compositions of piecewise linear functions is called \emph{piecewise linear network}.
\end{definition}
Since the concatenation and the composition of piecewise linear functions is clearly piecewise linear function, the output of any node in the piecewise linear network is written as a piecewise linear function of an input vector $\bm X$. 
This is also true for the saliency map function $\cA_i(\bm X), i \in [n]$. 
Furthermore, as discussed in \S4, we can focus on the input vector in the form of $\bm X(z) = \bm a_{1:n} + \bm b_{1:n} z$ which is parametrized by a scalar parameter $z \in \RR$.
Therefore, the saliency map value for each element is written as a piecewise linear function of the scalar parameter $z$, i.e., 
\begin{align} \label{eq:A_i_function_z}
 \cA_i(\bm X(z)) = 
 \begin{cases}
  \kappa_1^{\cA_i} z + \rho_1^{\cA_i}, & \text{if} ~~  z \in [L_1^{\cA_i}, U_1^{\cA_i}], \\ 
  \kappa_2^{\cA_i} z + \rho_2^{\cA_i}, & \text{if} ~~  z \in [L_2^{\cA_i}, U_2^{\cA_i}], \\ 
  \quad \quad \quad  \vdots \\ 
  \kappa_{K(\cA_i)}^{\cA_i} z + \rho_{K(\cA_i)}^f, & \text{if} ~~ z \in [L_{K(\cA_i)}^{\cA_i}, U_{K(\cA_i)}^{\cA_i}],
 \end{cases}
\end{align}
where
$K(\cA_i)$
is the number of linear pieces of the piecewise linear function, 
$\kappa_k^{\cA_i}, \rho_k^{\cA_i}$
are certain scalar parameters,
$[L_k^{\cA_i}, U_k^{\cA_i}]$
are intervals
for
$k \in [K(\cA_i)]$
(note that a polytope in $\RR^n$ is reduced to an interval when it is projected onto one-dimensional space).

This means that, for each piece of the piecewise linear function, we can identify the interval of $z$ such that $\cA_i(\bm X(z)) \ge \tau$ as follows~\footnote{For simplicity, we omit the description for the case of $\kappa_k^{\cA_i}=0$. In this case, if $\rho_k^{\cA_i} \ge \tau$, then $z \in [L_k^{\cA_i}, U_k^{\cA_i}] \Rightarrow i \in \cM_{\bm X(z)}$.}
\begin{align}
 \label{eq:z_interval}
 z \in
 \mycase{
 \left[
 \max\left(
 L_k^{\cA_i}, \left(\tau - \rho_k^{\cA_i} \right)/\kappa_k^{\cA_i}
 \right), \;
 U_k^{\cA_i}
 \right] & \text{ if } \kappa_k^{\cA_i} > 0
 \\
 \left[
 L_k^{\cA_i},
 \min\left(
 U_k^{\cA_i}, \left(\tau - \rho_k^{\cA_i} \right)/\kappa_k^{\cA_i}
 \right),
 \right] & \text{ if } \kappa_k^{\cA_i} < 0
 }
 \quad
 \Rightarrow
 \quad
 \cA_i(\bm X(z)) \ge \tau.
\end{align}
With a slight abuse of notation, let us collectively denote the finite number of intervals on $z \in \RR$ that are defined by $L_k^{\cA_i}, U_k^{\cA_i}, (\tau-\rho_i^{\cA_i}/\kappa_k^{\cA_i})$ for all $(k, i) \in [K(\cA_i)] \times[n]$ as 
\begin{align*}
 [z_0, z_1], [z_1, z_2], \ldots, [z_{t-1}, z_t], [z_t, z_{t+1}], \ldots, [z_{T-1}, z_T], 
\end{align*}
where
$z_{\rm min} = z_0$
and 
$z_{\rm max} = z_T$
are defined such that the probability mass of
$z < z_{\rm min}$
and 
$z > z_{\rm max}$
are negligibly small. 

\paragraph{Algorithm}
Algorithm~\ref{alg:si_dnn_saliency} shows how we identify
$\cZ = \{z \in \RR \mid \cM_{\bm X(z), \bm X^{\rm ref}(z)} = \cM_{\rm obs}\}$.
We simply check the intervals of $z$ in the order of
$[z_0, z_1], [z_1, z_2], ... , [z_{T-1}, z_T]$
to see whether 
$\cM_{\bm X(z)} = \cM_{\bm X(z_{\rm obs})}$
or not
in the interval 
by using Eq.\eq{eq:z_interval}.
Then, the truncation region $\cZ$ in Eq.\eq{eq:cZ} is given as
%\begin{align*}
$
 \cZ = \bigcup_{t \in [T] \mid \cE_{\bm X(z), \bm X^{\rm ref}(z)}  = \cE_{\rm obs} \text{ for } z \in [z_t, z_{t+1}]} [z_t, z_{t+1}]. 
%\end{align*}
$

\begin{algorithm}[!t]
\renewcommand{\algorithmicrequire}{\textbf{Input:}}
\renewcommand{\algorithmicensure}{\textbf{Output:}}
\begin{scriptsize}
 \begin{algorithmic}[1]
  \REQUIRE ${\bm X}^{\rm obs}, z_{\rm min}, z_{\rm max}$, $\cT \leftarrow \emptyset$
  \vspace{2pt}
  \STATE Obtain $\cE_{\rm obs}$, compute $\bm \eta$ as well as $\bm a$ and $\bm b$  $\leftarrow$ Eq. (\ref{eq:a_b_line}), and initialize: $t = 1$, $z_t=z_{\rm min}$	\vspace{2pt}
  \FOR {$t \le T$}
%  \vspace{2pt}
%  \STATE Compute $\bm \eta$ as well as $\bm a$ and $\bm b$  $\leftarrow$ Eq. (\ref{eq:a_b_line})
  \vspace{2pt}
%  \STATE Initialization: $t = 1$, $z_t=z_{\rm min}$	\vspace{2pt}
%  \FOR {$t \le T$}
  \vspace{2pt}
  \STATE Compute $z_{t+1}$ by Auto-Conditioning (see \S\ref{sec:implementation})
  \vspace{3pt}
  \IF{$\cE_{\bm X(z), \bm X^{\rm ref}(z)} = \cE_{\rm obs}$ in $z \in [z_t, z_{t+1}]$ (by using Eq.\eq{eq:z_interval})}
  \STATE $\cT \leftarrow \cT + \{t\}$
  \ENDIF
  %\STATE Check if $\cE_{\bm X(z), \bm X^{\rm ref}(z)} = \cE_{\rm obs}$ or not in $z \in [z_t, z_{t+1}]$ by using Eq.\eq{eq:z_interval}.
  \vspace{3pt}
  \STATE $t = t + 1$
  %	\vspace{1pt}
  \ENDFOR
  \vspace{2pt}
  \STATE Identify $\cZ \leftarrow \bigcup_{t \in \cT} [z_t, z_{t+1}]$
  \vspace{2pt}
  \STATE $p_{\rm selective} \leftarrow$ Eq. (\ref{eq:selective_p_parametrized})
  \vspace{2pt}
  \ENSURE $p_{\rm selective}$ 
 \end{algorithmic}
\end{scriptsize}
\caption{{\tt SI\_DNN\_Saliency}}
\label{alg:si_dnn_saliency}
\end{algorithm}

\section{Implementation: Auto-Conditioning} \label{sec:implementation}
%
% TODO : decreaseの意味を考える
The bottleneck of our algorithm is Line 3 in Algorithm~\ref{alg:si_dnn_saliency}, where $z_{t+1}$ must be found by considering all relevant piecewise linear components in a complicated trained CNN. 
The difficulty lies not only in the computational cost but also in the implementation cost. To implement conditional SI in DNNs naively, it is necessary to characterize all operations at each layer of the network as selection events and implement each of the specifically\citep{duy2020quantifying}
To circumvent this difficulty, we introduce a modular implementation scheme called \emph{auto-conditioning}, which is similar to \emph{auto-differentiation}~\citep{baydin2018automatic} in concept. 
This enables us to conduct conditional SI for a wide class of CNNs without additional implementation cost.

The basic idea in auto-conditioning is to add a mechanism to compute and maintain the interval $z \in [L^f_k, U^f_k]$ %where the input $\bm X(z)$ decreases 
for each piecewise linear component $f$ in the network (e.g., layer API in the Keras framework).
This enables us to automatically compute the interval $[L^f_k, U^f_k]$ of a piecewise linear function $f$ when it is obtained as concatenation and/or composition of multiple piecewise linear components.
If $f$ is obtained by concatenating two piecewise linear functions $f_1$ and $f_2$, we can easily obtain $[L^f_k, U^f_k] = [L^{f_1}_{k_1}, U^{f_1}_{k_1}] \cap [L^{f_2}_{k_2}, U^{f_2}_{k_2}]$.
However, if $f$ is obtained as a composition of two piecewise linear functions $f_1$ and $f_2$, the calculation of the interval is given by the following lemma.
\begin{lemma} \label{lemm:composition}
 Consider the composition of two piecewise linear functions, that is, $f(\bm X(z)) = (f_2 \circ f_1) (\bm X(z))$.
 Given a real value of $z$, the interval $[L_k^{f_2}, U_k^{f_2}]$ in the input domain of $f_2$
 % where $\bm X(z)$ decreases, 
 can be computed as 
 \begin{align*}
  L^{f_2}_{k_2}  
  & = 
  \max \limits_{j : (\Delta^{f_2}_{k_2} \bm \gamma^{f_1})_j < 0}
  \frac{
  (\bm \delta^{f_2}_{k_2})_j - 
  (\Delta^{f_2}_{k_2} \bm \beta^{f_1})_j
  }{
  (\Delta^{f_2}_{k_2} \bm \gamma^{f_1})_j
  },
  \quad \quad  
  U^{f_2}_{k_2} 
  =
  \min \limits_{j : (\Delta^{f_2}_{k_2} \bm \gamma^{f_1})_j > 0}
  \frac{
  (\bm \delta^{f_2}_{k_2})_j - 
  (\Delta^{f_2}_{k_2} \bm \beta^{f_1})_j
  }{
  (\Delta^{f_2}_{k_2} \bm \gamma^{f_1})_j
  }, 
 \end{align*}
 where 
 $\bm \beta^{f_1} +  \bm \gamma^{f_1} z$ is the output of $f_1$ (i.e., the input of $f_2$). Moreover, $\Delta^{f_2}_{k_2}$ and $\bm \delta^{f_2}_{k_2}$ are obtained by verifying the value of $\bm \beta^{f_1} +  \bm \gamma^{f_1} z$. 
 Then, 
 the interval of the composite function is obtained as follows:
 $[L^f_k, U^f_k] = [L^{f_1}_{k_1}, U^{f_1}_{k_1}] \cap [L^{f_2}_{k_2}, U^{f_2}_{k_2}]$
\end{lemma}

The proof is provided in Appendix \ref{app:proof_lemma_composition}. 
Here, the variables $\bm \beta^{f_k}$ and $\bm \gamma^{f_k}$ can be recursively computed through layers as
\begin{align*}
 \bm \beta^{f_{k+1}} = \Psi_k^{f_k} \bm \beta^{f_k} + \bm \psi_k^{f_k}
 ~~~ \text{and} ~~~
 \bm \gamma^{f_{k+1}} = \Psi_k^{f_k} \bm \gamma^{f_k}. 
\end{align*}
Lemma~\ref{lemm:composition} indicates that the intervals in which $\bm X(z)$ decreases can be \emph{forwardly propagated} through these layers. 
This means that the lower bound $L_k^{\cA_i}$ and upper bound $U_k^{\cA_i}$ of the current piece in the piecewise linear function in Eq. \eq{eq:A_i_function_z} can be automatically computed by \emph{forward propagation} of the intervals of the relevant piecewise linear components. 

\section{Experiment} \label{sec:experiment}

We only highlight the main results. More details (methods for comparison, network structure, etc.) can be found in the Appendix \ref{app:experimental_details}.

\paragraph{Experimental setup.}
We compared our proposed method with the naive method, over-conditioning (OC) method, and Bonferroni correction.
To investigate the false positive rate (FPR) we considerd, 1000 null images 
$\bm X = (X_1, ..., X_n)$ 
and 1000 reference images 
$\bm X^{\rm ref} = (X^{\rm ref}_1, ..., x^{\rm ref}_n)$, where $\bm s = \bm s^{\rm ref} = \bm 0$
and $\bm \veps, \bm \veps^{\rm ref} \sim \NN(\bm 0, I_n)$,
 for each $n \in \{ 64, 256, 1024, 4096\} $.
 To investigate the true positive rate (TPR), we set $n = 256$ and generated 1,000 images, in which $s_i = {\rm signal}$ for any $i \in \cS$ where $\cS$ is the ``true'' salient region whose location is randomly determined. $s_i = 0$ for any $i \not \in \cS$ and $\bm \veps \sim \NN(\bm 0, I_n)$.
We set $\Delta \in \{1, 2, 3, 4\}$.
Reference images were generated in the same way as in the case of FPR.
In all experiments, we set the threshold for selecting the salient region $\tau = 0$ in the mean null test and $\tau = 5$ in the global null test . 
We set the significance level $\alpha = 0.05$.
%
%For each case of checking the performance of FPR and TPR, we show two plots: one for global null test and another for mean null test.
%
We used CAM as the saliency method in all experiments.

\paragraph{Numerical results.}
The results of FPR control are presented in Fig. ~\ref{fig:fig_fpr}.
The proposed method, OC, and Bonferroni successfully controlled the FPR in both the mean and global null test cases, whereas the others could not.
Because naive methods failed to control the FPR, we no longer considered their TPR. 
The results of the TPR comparison are shown in Fig. ~\ref{fig:fig_tpr}.
The proposed method has the highest TPR in all cases.
The Bonferroni method has the lowest TPR because it is conservative owing to considering the number of all possible hypotheses.
The OC method also has a low TPR than the proposed method because it considers several extra conditions, which causes the loss of TPR.

\begin{figure}[!t]
\begin{minipage}{0.49\textwidth}
\begin{subfigure}{.49\linewidth}
  \centering
  \includegraphics[width=\linewidth]{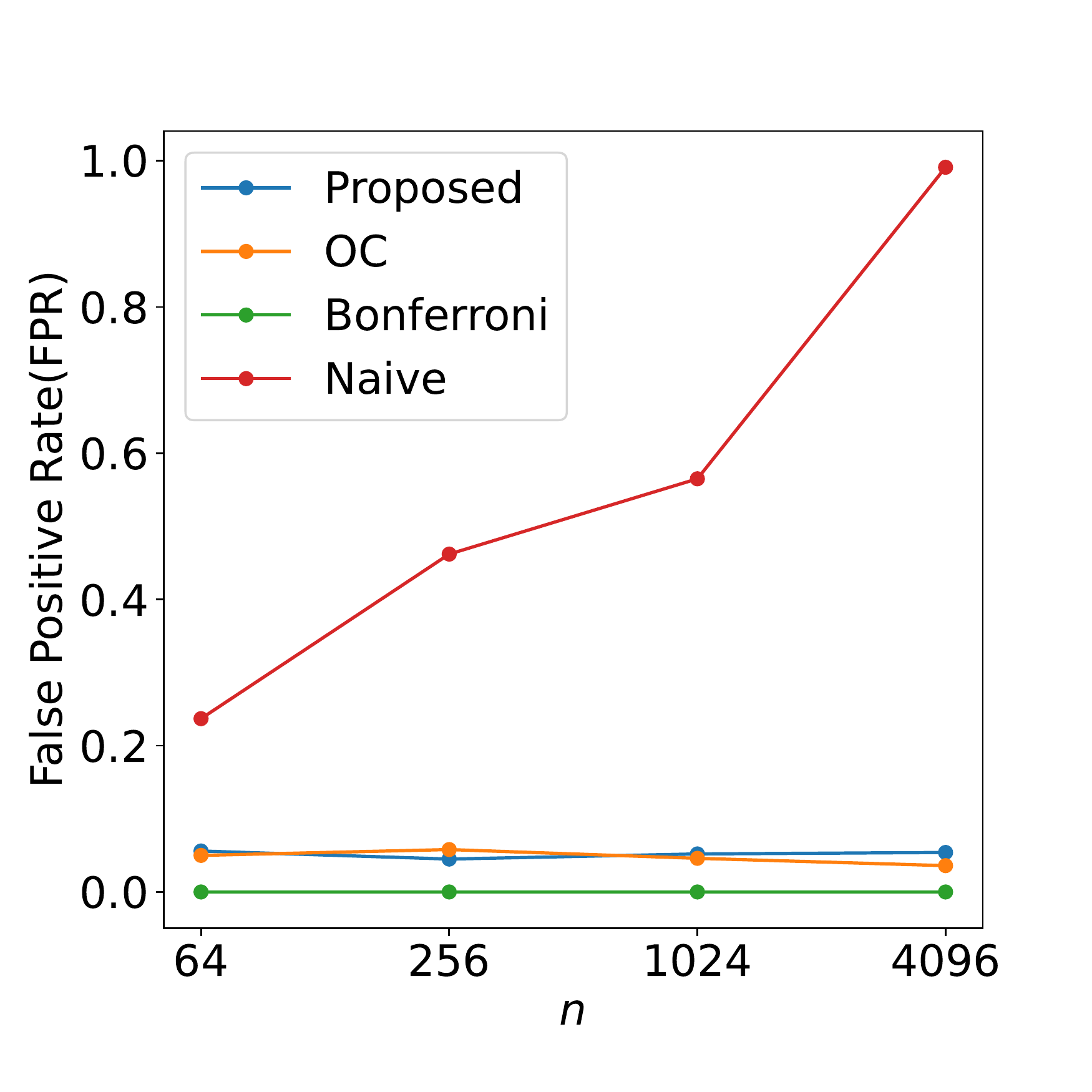}  
  \caption{Mean null test}
\end{subfigure}
\begin{subfigure}{.49\linewidth}
  \centering
  \includegraphics[width=\linewidth]{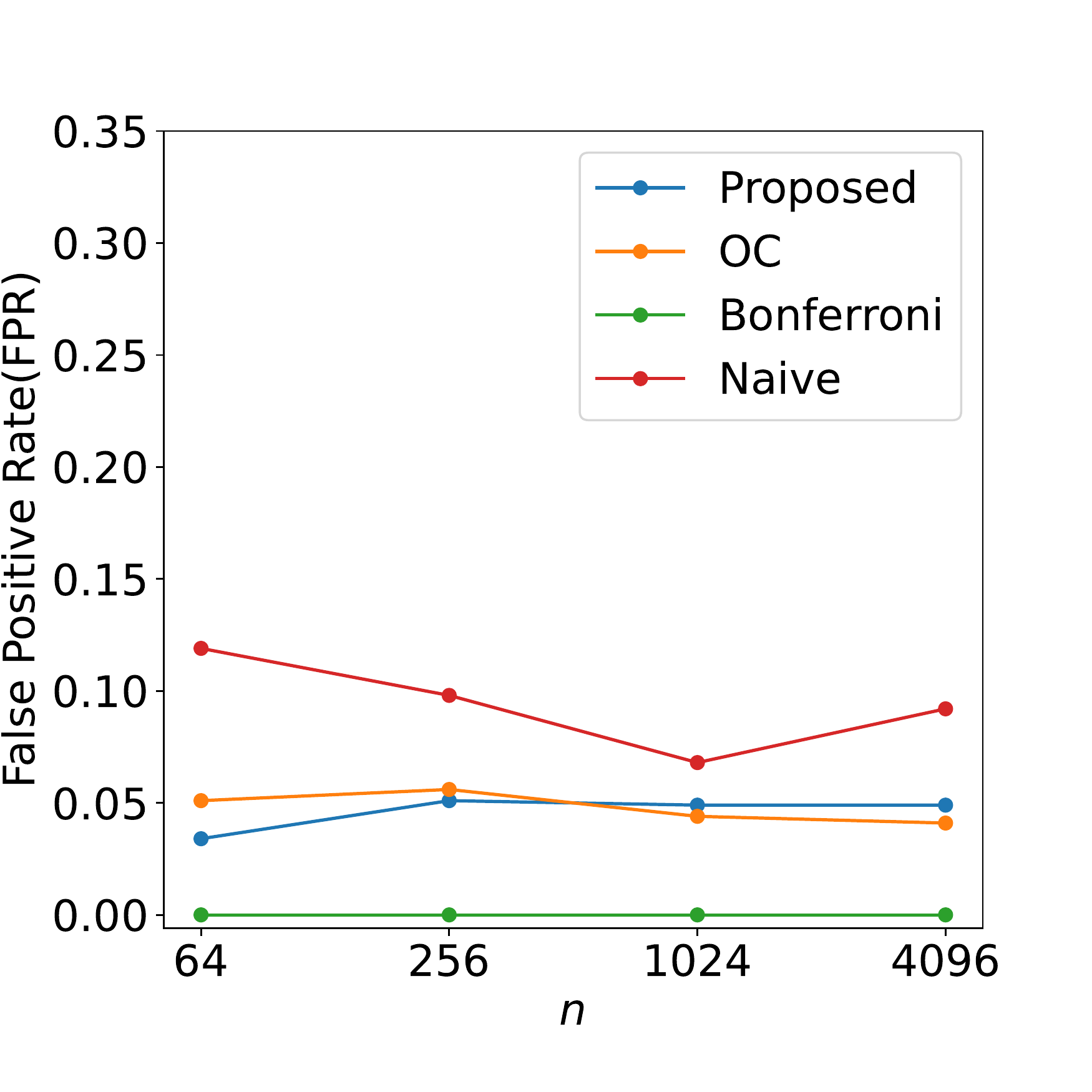}  
  \caption{Global null test}
\end{subfigure}
\caption{False Positive Rate (FPR) comparison.}
\label{fig:fig_fpr}
\end{minipage}\hfill
\begin{minipage}{0.49\textwidth}
\begin{subfigure}{.49\linewidth}
  \centering
  \includegraphics[width=\linewidth]{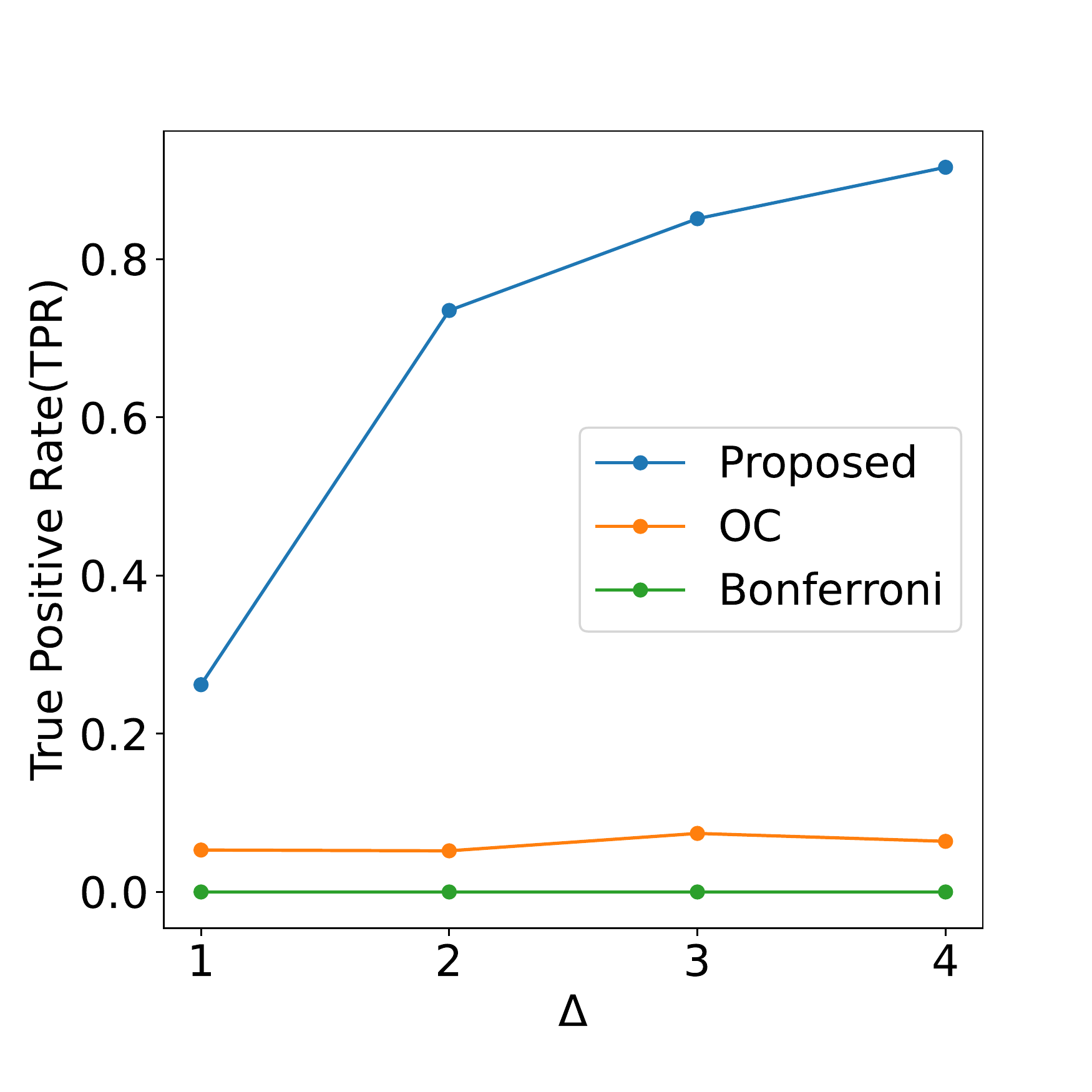}  
  \caption{Mean null test}
\end{subfigure}
\begin{subfigure}{.49\linewidth}
  \centering
  \includegraphics[width=\linewidth]{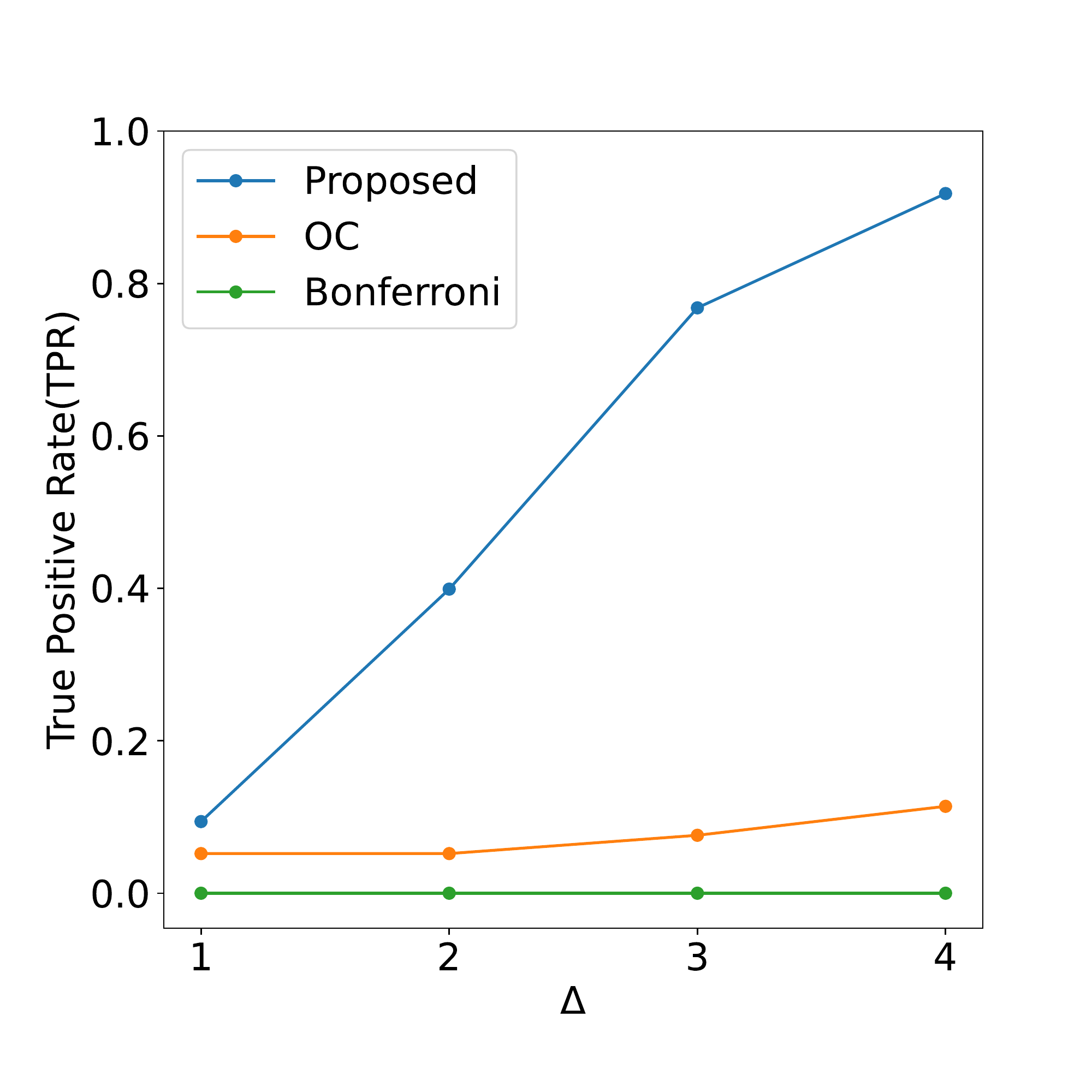}  
  \caption{Global null test}
\end{subfigure}
\caption{True Positive Rate (FPR) comparison.}
\label{fig:fig_tpr}
\end{minipage}
\end{figure}

\paragraph{Real data experiments.} We examined the brain image dataset extracted from the dataset used in \citet{buda2019association}, which included 939 and 941 images with and without tumors, respectively.
%
%We selected $50$ images from {\tt C1} as reference images.
%%
%We used $841$ images from {\tt C1} and $889$ images from {\tt C2} for DNN training.
%%
%The remaining images from {\tt C1} and {\tt C2} are used for demonstrating the advantages of the proposed selective $p$-value.
%
%The comparison results are shown in Table \ref{tab:real_world}.
%
The results of the mean null test are presented in Figs. \ref{fig:real_demo_fp_mean} and \ref{fig:real_demo_tp_mean}.
The results of the global null test are presented in Figs. \ref{fig:real_demo_fp_global} and \ref{fig:real_demo_tp_global}.
The naive $p$-value remains small even when the image has no tumor region, which indicates that naive $p$-values cannot be used to quantify the reliability of DNN-based salient regions.
The proposed method successfully identified false positive and true positive detections.

%\begin{table}[!t]
%\caption{False positive rate and power comparisons in brain image dataset.}
%\label{tab:real_world}
%\begin{center}
%\renewcommand{\arraystretch}{1.1}
%\begin{tabular}{ |c|c|c|c|c|c| } 
%         \hline
%            & \textbf{Naive} & \textbf{Permutation} & \textbf{Bonferroni} & \textbf{OC} & \textbf{Proposed} \\ 
%         \hline
%	\textbf{FPR}  &  & & & & \\ 
%        \hline
%         \textbf{TPR}  &  &  & & & \\ 
%         \hline
%        \end{tabular}
%\end{center}
%\end{table}
\begin{figure*}[!t]
\centering
%\begin{subfigure}{\textwidth}
%  \centering
  \includegraphics[width=.9\linewidth]{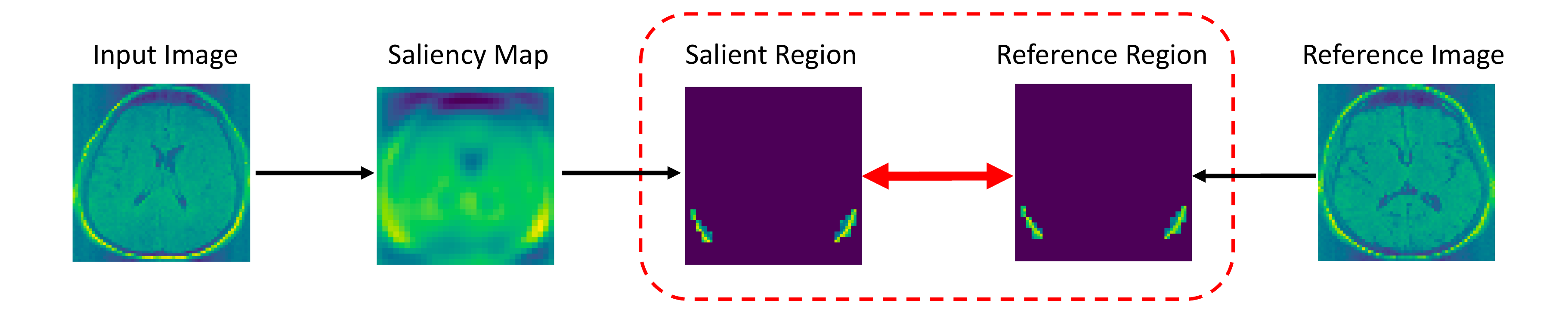}  
%  \caption{$p_{\rm naive}$ = \textbf{0.00}, $p_{\rm selective}$ = \textbf{0.78}}
%\end{subfigure}
%\begin{subfigure}{\textwidth}
%  \centering
%  \includegraphics[width=.95\linewidth]{real-data-mean-fp-2.pdf}  
%  \caption{$p_{\rm naive}$ = \textbf{0.01}, $p_{\rm selective}$ = \textbf{0.47}}
%\end{subfigure}
\caption{Mean null test for image without tumor ($p_{\rm naive} = \mathbf{0.00}$, $p_{\rm selective} = \mathbf{0.78}$).}
\label{fig:real_demo_fp_mean}
\end{figure*}

\begin{figure*}[!t]
\centering
%\begin{subfigure}{\textwidth}
%  \centering
  \includegraphics[width=.9\linewidth]{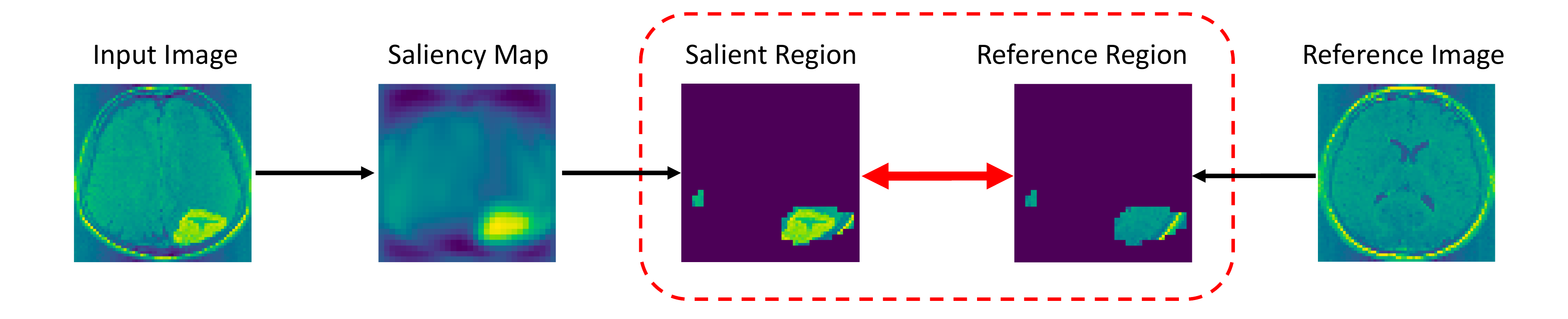}  
%  \caption{$p_{\rm naive}$ = \textbf{0.00}, $p_{\rm selective}$ = \textbf{1.92e-4}}
%\end{subfigure}
%\begin{subfigure}{\textwidth}
%  \centering
%  \includegraphics[width=.95\linewidth]{real-data-mean-tp-2.pdf}  
%  \caption{$p_{\rm naive}$ = \textbf{0.00}, $p_{\rm selective}$ = \textbf{2.82e-4}}
%\end{subfigure}
  \caption{Mean null test for image with a tumor ($p_{\rm naive} = \mathbf{0.00}$, $p_{\rm selective} = \mathbf{1.92\times10^{-4}}$).}
\label{fig:real_demo_tp_mean}
\end{figure*}

\begin{figure*}[!t]
\centering
%\begin{subfigure}{\textwidth}
%  \centering
  \includegraphics[width=.9\linewidth,page=1]{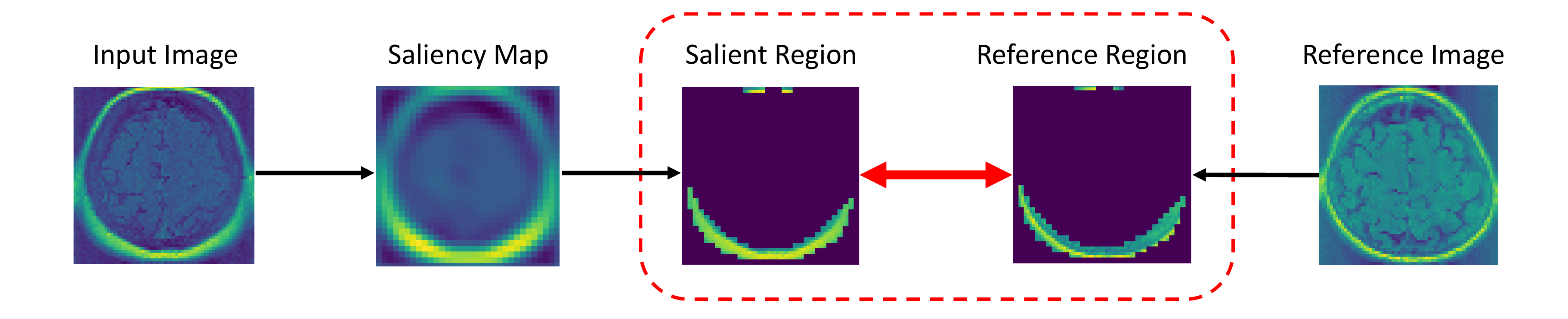}  
%  \caption{$p_{\rm naive}$ = \textbf{0.00}, $p_{\rm selective}$ = \textbf{1.00}}
%\end{subfigure}
%\begin{subfigure}{\textwidth}
%  \centering
%  \includegraphics[width=.95\linewidth]{real-data-global-fp-2.pdf}  
%  \caption{$p_{\rm naive}$ = \textbf{0.00}, $p_{\rm selective}$ = \textbf{0.94}}
%\end{subfigure}
\caption{Global null test for image without tumor ($p_{\rm naive} = \mathbf{0.03}$, $p_{\rm selective} = \mathbf{0.46}$)}
\label{fig:real_demo_fp_global}
\end{figure*}

\begin{figure*}[!t]
\centering
%\begin{subfigure}{\textwidth}
%  \centering
  \includegraphics[width=.9\linewidth,page=3]{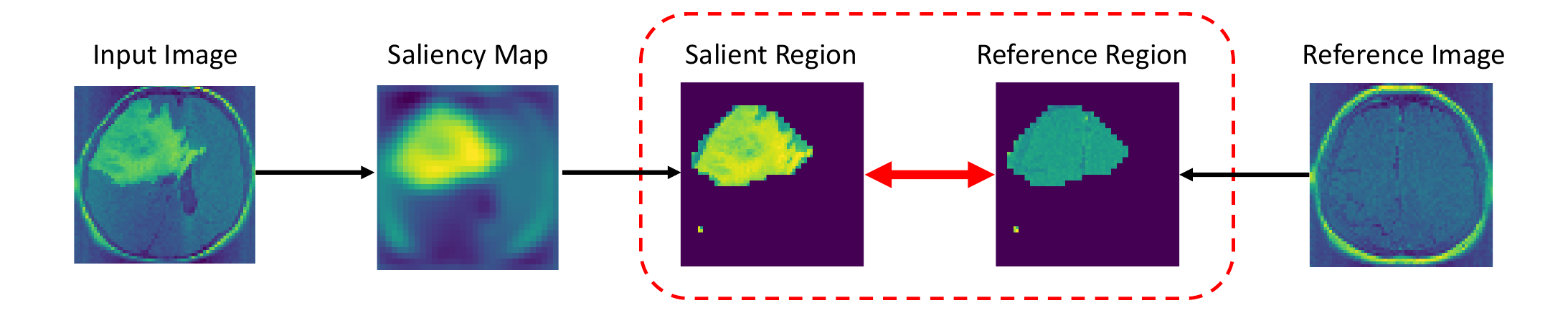}  
%  \caption{$p_{\rm naive}$ = \textbf{0.00}, $p_{\rm selective}$ = \textbf{3.01e-5}}
%\end{subfigure}
%\begin{subfigure}{\textwidth}
%  \centering
%  \includegraphics[width=.95\linewidth]{real-data-global-tp-2.pdf}  
%  \caption{$p_{\rm naive}$ = \textbf{0.00}, $p_{\rm selective}$ = \textbf{2.60e-6}}
%\end{subfigure}
  \caption{Global null test for image with a tumor ($p_{\rm naive} = \mathbf{0.00}$, $p_{\rm selective} = \mathbf{1.51\times 10^{-3}}$).}
\label{fig:real_demo_tp_global}
\end{figure*}

\section{Conclusion} \label{sec:conclusion}

In this study, we proposed a novel method to conduct statistical inference on the significance of DNN-driven salient regions based on the concept of conditional SI.
We provided a novel algorithm for efficiently and flexibly conducting conditional SI for salient regions.
We conducted experiments on both synthetic and real-world datasets to demonstrate the performance of the proposed method.

% --------------- Acknowledgement --------------------

\subsection*{Acknowledgements}
 This work was partially supported by MEXT KAKENHI (20H00601), JST CREST (JPMJCR21D3), JST Moonshot R\&D (JPMJMS2033-05), JST AIP Acceleration Research (JPMJCR21U2), NEDO (JPNP18002, JPNP20006), and RIKEN Center for Advanced Intelligence Project.

% --------------- Reference --------------------

\bibliographystyle{abbrvnat}
\bibliography{ref}

% --------------- Appendix --------------------

\appendix

\section{Appendix}

\subsection{Proof of Lemma \ref{lemma:data_line}} \label{app:proof_lemma_1}

In the mean null test, according to the second condition in \eq{eq:conditional_data_space}, we have 
\begin{align*}
	\Omega_{\bm X, \bm X^{\rm ref}} &= \Omega_{\bm{X}_{\rm{obs}},\bm{X}_{\rm{obs}}^{\rm{ref}}}  \\ 
	\Leftrightarrow ~
	\left (
	I_{2n} - 
	\frac{
        \bm \eta_{\cM_{\bm X}} \bm \eta_{\cM_{\bm X}}^\top
	} { 
	\bm \eta_{\cM_{\bm X}}^\top   \bm \eta_{\cM_{\bm X}}
	} 
	\right ) 
	{ \bm X \choose \bm X^{\rm ref}}
	& = 
	\Omega_{\bm{X}_{\rm{obs}},\bm{X}_{\rm{obs}}^{\rm{ref}}} \\ 
	\Leftrightarrow ~
	{ \bm X \choose \bm X^{\rm ref}}
	&= 
    \Omega_{\bm{X}_{\rm{obs}},\bm{X}_{\rm{obs}}^{\rm{ref}}}
    + \frac{ \bm \eta_{\cM_{\bm X}}}{\lVert \bm \eta_{\cM_{\bm X}}\rVert^2}
	\bm \eta_{\cM_{\bm X}}^\top  
	{ \bm X \choose \bm X^{\rm ref}}.
\end{align*}
By defining 
$\bm a = \bm{q}_{\bm{X}_{\rm{obs}},\bm{X}_{\rm{obs}}^{\rm{ref}}} $,
$\bm b = \frac{ \bm \eta_{\cM_{\bm X}}} {\lVert \bm{\eta_{\cM_{\bm X}}} \rVert^2}$,
$z = \bm \eta_{\cM_{\bm X}}^\top  {\bm X \choose \bm X^{\rm ref}} $, we obtain the result in Lemma \ref{lemma:data_line}. 

In the global null test, according to the second condition in \eq{eq:conditional_data_space},

\begin{align*}
    \mathcal{U}_{\bm X, \bm X^{\rm ref}} 
    &= \mathcal{U}_{\bm{X}_{\rm{obs}},\bm{X}_{\rm{obs}}^{\rm{ref}}} \\ 
    \Leftrightarrow  P^\perp_{\cM_{\bm X}} { \bm X \choose \bm X^{\rm ref}}
    &= \mathcal{U}_{\bm{X}_{\rm{obs}},\bm{X}_{\rm{obs}}^{\rm{ref}}} \\ 
    \Leftrightarrow \left( I_{2_n} - P_{\cM_{\bm X}} \right) { \bm X \choose \bm X^{\rm ref}}
    &= \mathcal{U}_{\bm{X}_{\rm{obs}},\bm{X}_{\rm{obs}}^{\rm{ref}}} \\ 
    \Leftrightarrow { \bm X \choose \bm X^{\rm ref}} 
    &= \mathcal{U}_{\bm{X}_{\rm{obs}},\bm{X}_{\rm{obs}}^{\rm{ref}}} +
    \mathcal{V}_{\bm{X}_{\rm{obs}},\bm{X}_{\rm{obs}}^{\rm{ref}}}
    \sigma^{-1}\left\lVert P_{\cM_{\bm X}}{ \bm X \choose \bm X^{\rm ref}} \right\rVert
.\end{align*}

By defining 
$\bm a = \mathcal{U}_{\bm{X}_{\rm obs},\bm{X}_{\rm obs}^{\rm ref}} $,
$\bm b = \mathcal{V}_{\bm{X}_{\rm{obs}},\bm{X}_{\rm obs}^{\rm ref}} $,
$z = \sigma^{-1}\left\lVert P_{\cM_{\bm X}}{ \bm X \choose \bm X^{\rm ref}} \right\rVert$
, we obtain the result in Lemma \ref{lemma:data_line}. 

\subsection{Examples of piecewise linear functions} \label{app:example_piecewise_linear}

Examples of piecewise linear components in a trained CNN with $\bm X \in \RR^2$ are provided as follows:

\noindent
{\it ReLU}:
Consider $f$ is ReLU function.
Then, $K(f) = 4$, $\bm \psi_k = (0 ~ 0)^\top $ for any $k \in [4]$,
\begin{align*}
	\Psi_1^f = 
	\begin{pmatrix}
		0 & 0 \\ 
		0 & 0
	\end{pmatrix},
	~ 
	\cP_1^f = \left \{ 
	\bm X: 
	\begin{array}{l}
		X_1 < 0, \\ X_2 < 0
	\end{array}
	\right \}, 
	\quad 
	\Psi_2^f = 
	\begin{pmatrix}
		0 & 0 \\ 
		0 & 1
	\end{pmatrix},
	~ 
	\cP_2^f = \left \{ 
	\bm X: 
	\begin{array}{l}
		X_1 < 0, \\ X_2 \geq 0
	\end{array}
	\right \},
	\\ 
	\Psi_3^f = 
	\begin{pmatrix}
		1 & 0 \\ 
		0 & 0
	\end{pmatrix},
	~ 
	\cP_3^f = \left \{ 
	\bm X: 
	\begin{array}{l}
		X_1 \geq 0, \\ X_2 < 0
	\end{array}
	\right \}, 
	\quad 
	\Psi_4^f = 
	\begin{pmatrix}
		1 & 0 \\ 
		0 & 1
	\end{pmatrix},
	~ 
	\cP_4^f = \left \{ 
	\bm X: 
	\begin{array}{l}
		X_1 \geq 0, \\ X_2 \geq 0
	\end{array}
	\right \}.
\end{align*}
This can be similarly extended to the case of Leaky ReLU.

\noindent
{\it Max-pooling}:
Consider $f (\bm X) = {\rm max}\{ X_1, X_2\}$.
Then, it is represented as a piecewise linear function with $K(f) = 2$, $\bm \psi_k = (0) $ for any $k \in [2]$,
\begin{align*}
	\Psi_1^f = 
	\begin{pmatrix}
		1 & 0 
	\end{pmatrix},
	~ 
	\cP_1^f = \left \{ 
	\bm X: X_1 \geq X_2
	\right \}, 
	\quad 
	\Psi_2^f = 
	\begin{pmatrix}
		0 & 1
	\end{pmatrix},
	~ 
	\cP_2^f = \left \{ 
	\bm X: X_1 < X_2
	\right \}.
\end{align*}

\noindent
{\it Convolution and matrix-vector multiplication}:
In a neural network, the multiplication results between the weight matrix and the output of the previous layer and its summation with the bias vector is a linear function. 
In a CNN, the convolution operation is obviously a linear function.

\noindent
{\it Upsampling}:
Consider $f$ is the upsampling operation on 
$\bm X \in \RR^2$, then it can be represented as a piecewise linear function with $K(f) = 1$, $\bm \psi_1 = (0 ~ 0 ~ 0 ~ 0)^\top$,
\begin{align*}
	\Psi_1^f = 
	\begin{pmatrix}
		1 & 1 & 0 & 0 \\ 
		0 & 0 & 1 & 1
	\end{pmatrix}^\top,
	\quad 
	\cP_1^f = \RR^2.
\end{align*}

\noindent
{\it Sigmoid and hyperbolic tangent}:
If there is any specific demand to use non-piecewise linear activation functions, we can apply a piecewise-linear approximation approach to these functions.

\subsection{Proof of Lemma \ref{lemm:composition}} \label{app:proof_lemma_composition}

At $f_1$, given a a real value $z$, the input is $\bm \beta^{f_0} + \bm \gamma^{f_0} z = \bm a_{1:n} + \bm b_{1:n} z$. 
By checking the value of this input, we can easily obtain the polytope 
\[
\{ \Delta^{f_1}_{k_1} (\bm \beta^{f_0} + \bm \gamma^{f_0} z) \leq \bm \delta^{f_1}_{k_1}\}, \quad k_1 \in [K(f_1)],
\] 
that $\bm \beta^{f_0} + \bm \gamma^{f_0} z$ belongs to.
Based on the obtained polytope, we can calculate the interval $[L^{f_1}_{k_1}, U^{f_1}_{k_1}]$,
\begin{align*}
	L^{f_1}_{k_1}  
	= 
	\max \limits_{j : (\Delta^{f_1}_{k_1} \bm \gamma^{f_0})_j < 0}
	\frac{
	(\bm \delta^{f_1}_{k_1})_j - 
	(\Delta^{f_1}_{k_1} \bm \beta^{f_0})_j
	}{
	(\Delta^{f_1}_{k_1} \bm \gamma^{f_0})_j
	}
	\quad 
	\text{and}
	\quad 
	U^{f_1}_{k_1} 
	=
	\min \limits_{j : (\Delta^{f_1}_{k_1} \bm \gamma^{f_0})_j > 0}
	\frac{
	(\bm \delta^{f_1}_{k_1})_j - 
	(\Delta^{f_1}_{k_1} \bm \beta^{f_0})_j
	}{
	(\Delta^{f_1}_{k_1} \bm \gamma^{f_0})_j
	}.
\end{align*}
Moreover, based on the obtained polytope, we can easily obtain $\Psi^{f_1}_{k_1}$ and $\bm \psi^{f_1}_{k_1}$, $k_1 \in [K(f_1)]$.
Therefore, the output of the first layer at $z$ can be defined as 
\begin{align*}
	f_{1} (z) 
	& = \Psi^{f_1}_{k_1} (\bm \beta^{f_0} + \bm \gamma^{f_0} z) + \bm \psi^{f_1}_{k_1} \\ 
	& = \bm \beta^{f_1} + \bm \gamma^{f_1} z,
\end{align*}
where $\bm \beta^{f_1} = \Psi^{f_1}_{k_1} \bm \beta^{f_0} + \bm \psi^{f_1}_{k_1} $ and $\bm \gamma^{f_1}  = \Psi^{f_1}_{k_1} \bm \gamma^{f_0}$.
Next, we input $\bm \beta^{f_1}$, $\bm \gamma^{f_1}$ to $f_2$.

At the $2^{\rm nd}$ layer, similarly, the input is $\bm \beta^{f_1} + \bm \gamma^{f_1} z$. 
By checking the value of this input, we can easily obtain the polytope 
\[
\{ \Delta^{f_2}_{k_2} (\bm \beta^{f_1} + \bm \gamma^{f_1} z) \leq \bm \delta^{f_2}_{k_2}\}, \quad k_2 \in [K(f_2)],
\] 
that $\bm \beta^{f_1} + \bm \gamma^{f_1} z$ belongs to.
Based on the obtained polytope, we can calculate the interval $[L^{f_2}_{k_2}, U^{f_2}_{k_2}]$,
\begin{align*}
	L^{f_2}_{k_2}  
	= 
	\max \limits_{j : (\Delta^{f_2}_{k_2} \bm \gamma^{f_1})_j < 0}
	\frac{
	(\bm \delta^{f_2}_{k_2})_j - 
	(\Delta^{f_2}_{k_2} \bm \beta^{f_1})_j
	}{
	(\Delta^{f_2}_{k_2} \bm \gamma^{f_1})_j
	}
	\quad 
	\text{and}
	\quad 
	U^{f_2}_{k_2} 
	=
	\min \limits_{j : (\Delta^{f_2}_{k_2} \bm \gamma^{f_1})_j > 0}
	\frac{
	(\bm \delta^{f_2}_{k_2})_j - 
	(\Delta^{f_2}_{k_2} \bm \beta^{f_1})_j
	}{
	(\Delta^{f_2}_{k_2} \bm \gamma^{f_1})_j
	}.
\end{align*}
Moreover, based on the obtained polytope, we can easily obtain $\Psi^{f_2}_{k_2}$ and $\bm \psi^{f_2}_{k_2}$, $k_2 \in [K(f_2)]$.
Therefore, the output of the first layer at $z$ can be defined as 
\begin{align*}
	f_{2} (z) 
	& = \Psi^{f_2}_{k_2} (\bm \beta^{f_1} + \bm \gamma^{f_1} z) + \bm \psi^{f_2}_{k_2} \\ 
	& = \bm \beta^{f_2} + \bm \gamma^{f_2} z,
\end{align*}
where $\bm \beta^{f_2} = \Psi^{f_2}_{k_2} \bm \beta^{f_1} + \bm \psi^{f_2}_{k_2} $ and $\bm \gamma^{f_2}  = \Psi^{f_2}_{k_2} \bm \gamma^{f_1}$.

\subsection{Experimental details.} \label{app:experimental_details}

\paragraph{Methods for comparison.} We compared our proposed method with the following approaches:

$\bullet$ Naive: the classical $z$-test is used to calculate the naive $p$-value.

$\bullet$ Bonferroni: the number of all possible hypotheses are considered to account for the selection bias.
The $p$-value is computed by $p_{\rm bonferroni} = \min (1, p_{\rm naive} * 2^n)$

$\bullet$ Over-conditioning (OC): additionally conditioning on the observed activeness and inactiveness of all the nodes. The limitation of this method is its low statistical power due to over-conditioning.

\paragraph{Network structure.} In all the experiments, we used the network structure shown in Fig. \ref{fig:network}.

\begin{figure*}[!t]
\centering
  \centering
  \includegraphics[width=.8\linewidth]{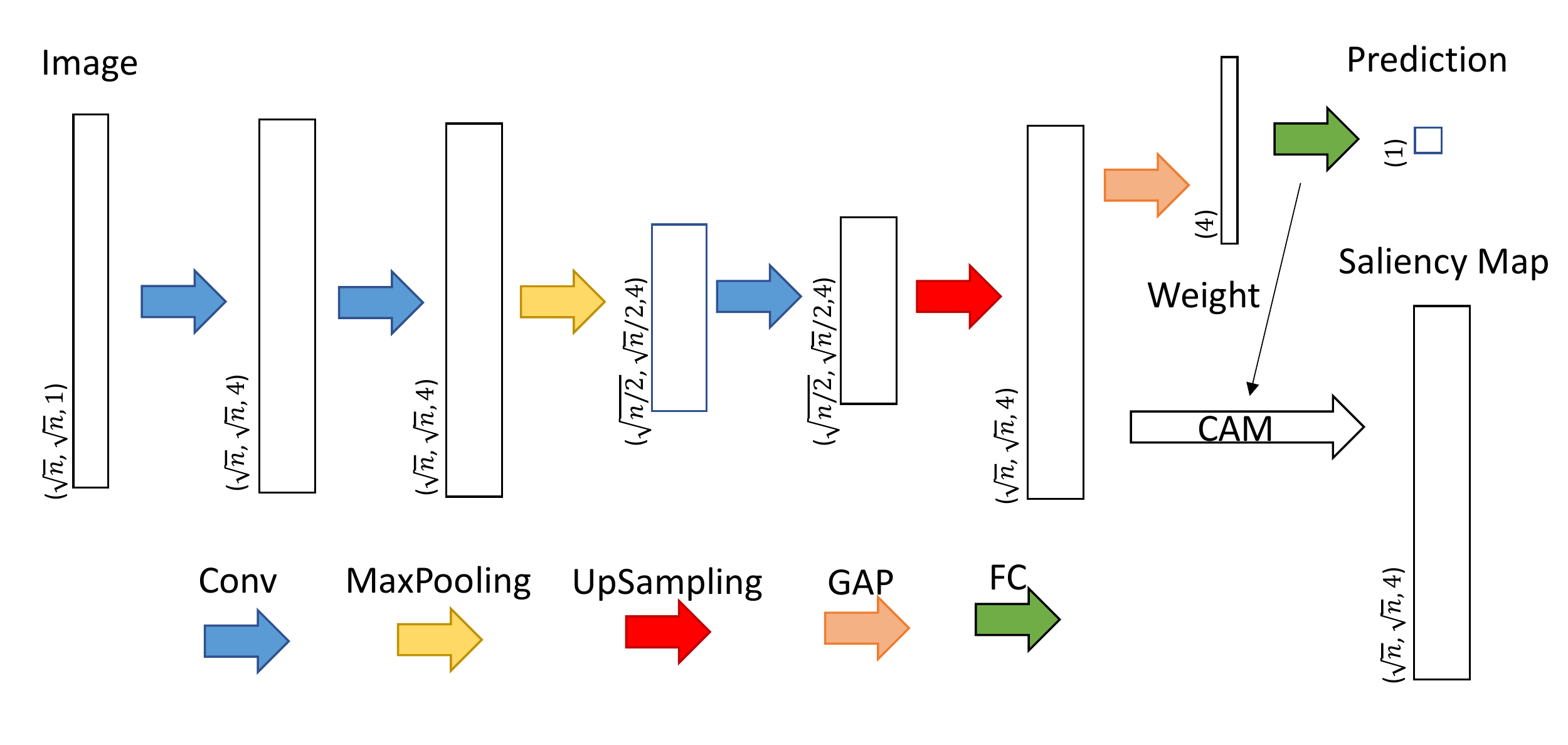}  
\caption{Network structure.}
\label{fig:network}
\end{figure*}

% \paragraph{Experiment with multiple reference images.}
% 
% %
% Additionally, we also considered the case where $M$ multiple reference images are given, which is an extension of the problem setup in \S \ref{sec:problem_setup} where a single ($M = 1$) reference image is given.
% %
% We set $M = 2$ and conducted the FPR and TPR experiments similarly to the case of $M = 1$ above.
% %
% The results are shown in Figs. \ref{fig:fpr_multi_reference} and \ref{fig:tpr_multi_reference}

%\begin{figure}[!t]
%\begin{minipage}{0.49\textwidth}
%\begin{subfigure}{.495\linewidth}
%  \centering
%  \includegraphics[width=\linewidth]{./img/mul_ref/fpr_abs.pdf}  
%  \caption{Global null test}
%\end{subfigure}
%\begin{subfigure}{.495\linewidth}
%  \centering
%  \includegraphics[width=\linewidth]{./img/mul_ref/fpr_avg.pdf}  
%  \caption{Mean null test}
%\end{subfigure}
%\caption{FPR (multiple reference images).}
%\label{fig:fpr_multi_reference}
%\end{minipage}\hfill
%\begin{minipage}{0.49\textwidth}
%\begin{subfigure}{.495\linewidth}
%  \centering
%  \includegraphics[width=\linewidth]{./img/mul_ref/tpr_abs.pdf}  
%  \caption{Global null test}
%\end{subfigure}
%\begin{subfigure}{.495\linewidth}
%  \centering
%  \includegraphics[width=\linewidth]{./img/mul_ref/tpr_avg.pdf}  
%  \caption{Mean null test}
%\end{subfigure}
%\caption{TPR (multiple reference images).}
%\label{fig:tpr_multi_reference}
%\end{minipage}
%\end{figure}

\paragraph{Experimental setting on brain image dataset.}
We examine the brain image dataset extracted from the dataset used in \cite{buda2019association}, which includes 941 images without tumors ({\tt C1}) and 939 images with tumors ({\tt C2}).
We selected $50$ images from {\tt C1} as reference images.
We used $841$ images from {\tt C1} and $889$ images from {\tt C2} for DNN training.
The remaining images from {\tt C1} and {\tt C2} are used for demonstrating the advantages of the proposed selective $p$-value.

\paragraph{More results on brain image dataset.} Additional results are shown in Figs. \ref{fig:real_demo_fp_global_more}, \ref{fig:real_demo_tp_global_more}, \ref{fig:real_demo_fp_mean_more} and \ref{fig:real_demo_tp_mean_more}

\clearpage

\begin{figure*}[!t]
\centering
%\begin{subfigure}{\textwidth}
%  \centering
%  \includegraphics[width=.95\linewidth]{real-data-mean-fp-1.pdf}  
%  \caption{$p_{\rm naive}$ = \textbf{0.00}, $p_{\rm selective}$ = \textbf{0.78}}
%\end{subfigure}
\begin{subfigure}{\textwidth}
  \centering
  \includegraphics[width=.95\linewidth]{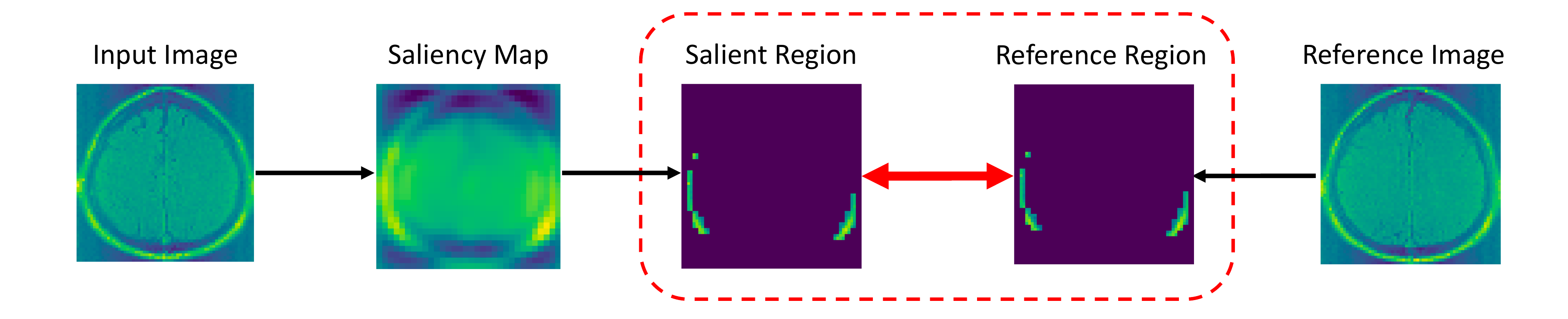}  
  \caption{$p_{\rm naive} = \mathbf{0.01}$, $p_{\rm selective} = \mathbf{0.47}$}
\end{subfigure}
\caption{Inference on salient regions for images without tumor (mean null test).}
\label{fig:real_demo_fp_mean_more}
\end{figure*}

\begin{figure*}[!t]
\centering
%\begin{subfigure}{\textwidth}
%  \centering
%  \includegraphics[width=.95\linewidth]{real-data-mean-tp-1.pdf}  
%  \caption{$p_{\rm naive}$ = \textbf{0.00}, $p_{\rm selective}$ = \textbf{1.92e-4}}
%\end{subfigure}
\begin{subfigure}{\textwidth}
  \centering
  \includegraphics[width=.95\linewidth]{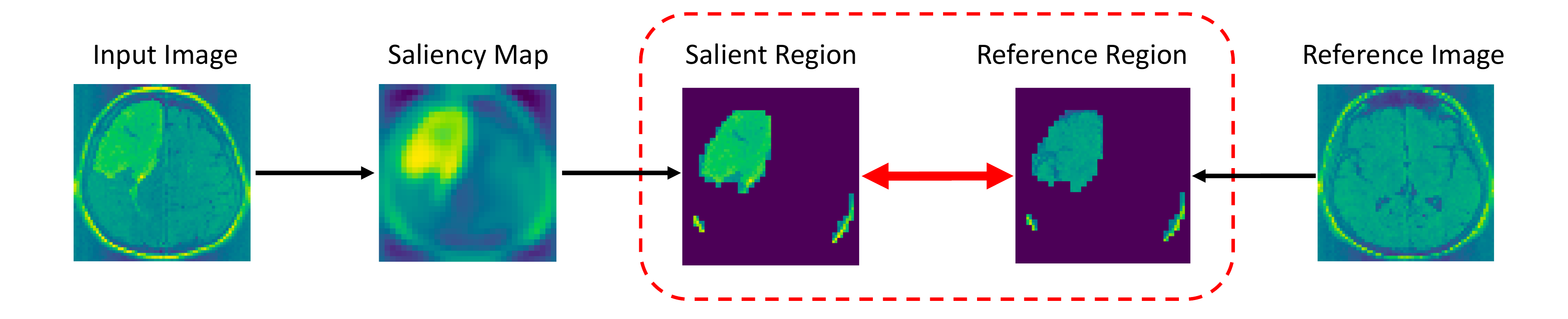}  
  \caption{$p_{\rm naive} = \mathbf{0.00}$, $p_{\rm selective} = \mathbf{2.82\times 10^{-4}}$}
\end{subfigure}
\caption{Inference on salient regions for images where there exists a tumor (mean null test).}
\label{fig:real_demo_tp_mean_more}
\end{figure*}

\begin{figure*}[!t]
\centering
%\begin{subfigure}{\textwidth}
%  \centering
%  \includegraphics[width=.95\linewidth]{real-data-global-fp-1.pdf}  
%  \caption{$p_{\rm naive}$ = \textbf{0.00}, $p_{\rm selective}$ = \textbf{1.00}}
%\end{subfigure}
\begin{subfigure}{\textwidth}
  \centering
  \includegraphics[width=.95\linewidth,page=2]{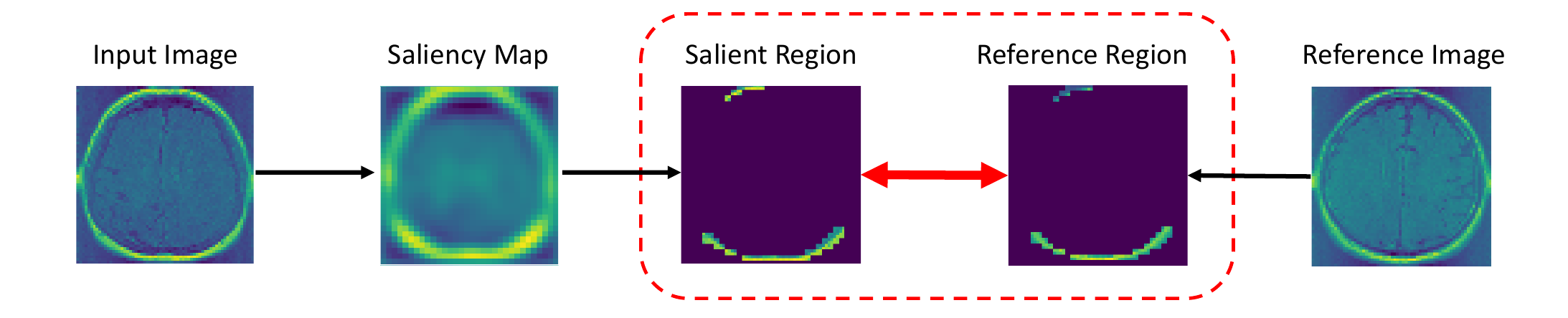}  
  \caption{$p_{\rm naive} = \mathbf{3.00\times{10}^{-4}}$, $p_{\rm selective} = \mathbf{0.29}$}
\end{subfigure}
\caption{Inference on salient regions for images without tumor (global null test).}
\label{fig:real_demo_fp_global_more}
\end{figure*}

\begin{figure*}[!t]
\centering
%\begin{subfigure}{\textwidth}
%  \centering
%  \includegraphics[width=.95\linewidth]{real-data-global-tp-1.pdf}  
%  \caption{$p_{\rm naive}$ = \textbf{0.00}, $p_{\rm selective}$ = \textbf{3.01e-5}}
%\end{subfigure}
\begin{subfigure}{\textwidth}
  \centering
  \includegraphics[width=.95\linewidth,page=4]{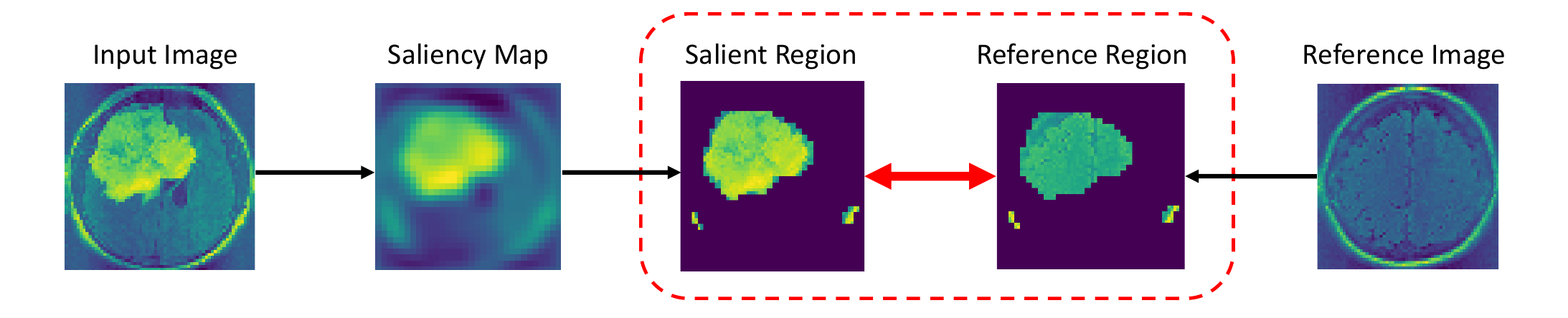}  
  \caption{$p_{\rm naive} = \mathbf{0.00}$, $p_{\rm selective} = \mathbf{2.66\times 10^{-20}}$}
\end{subfigure}
\caption{Inference on salient regions for images where there exists a tumor (global null test).}
\label{fig:real_demo_tp_global_more}
\end{figure*}

\clearpage

\end{document}